\documentclass[final]{alt2023}

\usepackage{booktabs}
\usepackage[font=small,labelfont=bf]{caption}
\usepackage{times}
\usepackage{wrapfig}
\usepackage{comment}

\newcommand{\Bb}{\mathcal B}
\newcommand{\C}{\mathbb C}
\newcommand{\Cc}{\mathcal C}

\newcommand{\kl}{\mathrm{kl}}
\newcommand{\KL}{\mathrm{K\!L}}
\newcommand{\E}{\mathbb{E}}

\newcommand{\kron}{\delta}
\newcommand{\Pp}{\mathbb{P}}
\newcommand{\PR}{\mathcal{P}}
\newcommand{\Q}{\mathcal{Q}}
\newcommand{\R}{\mathbb{R}}
\newcommand{\V}{\mathbb{V}}
\newcommand{\X}{\mathcal{X}}
\newcommand{\Y}{\mathcal{Y}}

\newcommand{\emme}{\mathfrak{m}}
\newcommand{\esse}{\mathfrak{s}}

\DeclareMathOperator*{\argmax}{argmax}

\DeclareMathOperator*{\diag}{diag}
\DeclareMathOperator*{\erf}{erf}
\DeclareMathOperator*{\Id}{Id}
\DeclareMathOperator*{\ReLU}{ReLU}

\newtheorem{manpropin}{Proposition}
\newenvironment{manualprop}[1]{%
  \renewcommand\themanpropin{#1}%
  \manpropin
}{\endmanpropin}

\newenvironment{proofsketch}{%
  \renewcommand{\proofname}{Proof's sketch}\proof}{\endproof}

\title[Wide stochastic networks]{Wide stochastic networks: Gaussian limit and PAC-Bayesian training}

\altauthor{%
 \Name{Eugenio Clerico} \Email{clerico@stats.ox.ac.uk}\\
 \Name{George Deligiannidis}
 \Email{deligian@stats.ox.ac.uk}\\
 \Name{Arnaud Doucet}
 \Email{doucet@stats.ox.ac.uk}\\
 \addr Department of Statistics, University of Oxford%
}

\begin{document}

\maketitle

\begin{abstract}%
The limit of infinite width allows for substantial simplifications in the analytical study of over-parameterised neural networks. With a suitable random initialisation, an extremely large network exhibits an approximately Gaussian behaviour. In the present work, we establish a similar result for a simple stochastic architecture whose parameters are random variables, holding both before and during training.
The explicit evaluation of the output distribution allows for a PAC-Bayesian training procedure that directly optimises the generalisation bound.
For a large but finite-width network, we show empirically on MNIST that this training approach can outperform standard PAC-Bayesian methods.
\end{abstract}

\begin{keywords}%
  Infinite width; Gaussian limit; PAC-Bayes; Stochastic networks.
\end{keywords}

\section{Introduction}
In recent years, overparameterised artificial neural networks with millions of nodes have shown  remarkably good generalisation capabilities. This behaviour contradicts the traditional well-rooted belief that overfitting is unavoidable when the trainable parameters far outnumber the size of the training dataset. It also highlights how the complexity bounds from classical statistical learning theory \citep{vapnik00,Bousquet2004, shalevBook2014} are manifestly inadequate tools to assess the generalisation properties of modern neural architectures \citep{zhang2017understanding}. As a consequence, the last couple of decades have seen the flourishing of novel results and techniques, aiming to explain the undeniable success of overparameterised models. 

A large number of trainable parameters makes the direct study of a network's training dynamics extremely challenging. However, things become more manageable when approximations are made, as is the case in the limit of infinite width \citep{neal, samuel, yang2019tensor_i, hayou, lee_wide_nn_ntk, sirignano2019mean, debortoli2020quantitative, hayou2021stable}. For a fully-connected feed-forward network, this limit consists in assuming that each layer includes an infinite number of nodes, while alternative definitions of \textit{width} allow for extensions of this idea to encompass a vast range of architectures \citep{yang2019tensor_i}. Although unachievable in practice, infinitely wide networks feature the interesting property of behaving like Gaussian processes at initialisation, when all the parameters are suitably randomly initialised. This fact enable us to capture the output's behaviour of large (but finite-size) models, both before \citep{matthews, lee_gaussian_process} and during the training \citep{jacot}. 

In this work, we establish a similar asymptotic result for a simple stochastic architecture, featuring a single hidden layer. For a \textit{stochastic} network, the randomness is not limited to the initialisation but is intrinsic in the parameters, which are treated as random variables. Specifically, here we assume that each parameter follows an independent normal distribution. As the architecture's width approaches infinity, we show that the network's output becomes Gaussian, with mean and covariance that can be derived from the means and standard deviations of the random parameters. We also show that under a lazy-regime assumption, where the parameters stay close to their initial values, this Gaussian behaviour is preserved throughout the training.

Part of the interest in studying stochastic networks is their role in the context of \textit{learning with guarantees}, where the goal is to provide an upper-bound on the generalisation error without making use of any held-out test dataset. For long, in the overparameterised regime tight bounds could only be achieved under strong, and often 
hardly verifiable, hypotheses \citep{allenzhu2020learning}. However, some promising non-vacuous results have been recently obtained by applying PAC-Bayesian methods to the training of stochastic classifiers \citep{dziugaite2017computing, Zhou2019NonvacuousGB, perezortiz2021tighter, pmlr-v162-biggs22a, clerico22aistats}.

The PAC-Bayesian theory originated from the seminal work of \cite{ShaweTaylor1997APA}, \cite{shawetaylor1998}, and \cite{mcall-pac-bayesian, mcallester}. We refer to \cite{catoniPAC} for an extensive monograph on the topic, and to \cite{guedj2019primer} and \cite{alquier2021user} for recent introductory overviews. It is a framework that provides upper bounds on the expected generalisation error of stochastic classifiers, with high probability over the random draw of the training dataset. The underlying idea is that if the distribution of the network's parameters does not change much during the training, then the learnt model should not be prone to overfit. 

We call PAC-Bayesian training a procedure that aims to optimise a PAC-Bayesian bound. Often this optimisation cannot be tackled directly, as the distribution of the network's output is unknown, and one needs to sample multiple realisations of the stochastic parameters \citep{dziugaite2017computing, perezortiz2021tighter}. In this paper, we propose to train a shallow wide stochastic network by exploiting the fact that it has an approximately Gaussian output. Notably, this approach allows for the direct optimisation of PAC-Bayesian bounds, even when a non-differentiable loss function, such as the 01-loss, is considered. We show empirically that the procedure that we present can bring tighter bounds and outperform standard PAC-Bayesian training methods. 

As a final remark, it is worth mentioning that this is not the first work suggesting to exploit the output's Gaussianity to train a stochastic network. For instance, \cite{alquier-ridgway-chopin} uses a similar approach, but limited to a linear model for binary classification. Also, \cite{clerico22aistats} built on a preprint of the current paper to develop a Gaussian training method for multilayer architectures. 
\section{Stochastic networks}
We consider a simple network $\R^p\to \R^q$, consisting of a single hidden layer made of $n$ nodes:
\begin{equation}\label{eq:net}F(x) = W^1\,\phi(W^0 x)\,,\end{equation}
where $W^1$ is a $q\times n$ matrix, $W^0$ a $n\times p$ matrix, and $\phi$ the activation function. The network is stochastic. This means that $W^0$ and $W^1$ are random variables and each time a new input is fed to the network a new realisation of them is used to evaluate the output. Concretely, we let 
$$W^1_{ij} = \tfrac{1}{\sqrt n}(\esse^1_{ij} \zeta^1_{ij} + \emme^1_{ij})\,;\qquad W^0_{jk} = \tfrac{1}{\sqrt p}(\esse^0_{jk} \zeta^0_{jk} + \emme^0_{jk})\,,$$
where $(\zeta^1_{ij})_{i=1\dots q}^{j=1\dots n}$ and $(\zeta^0_{jk})_{j=1\dots n}^{k=1\dots p}$ are independent families of iid standard normal random variables. We will henceforth call hyper-parameters the means $\emme$ and the standard deviations $\esse$, which are deterministic quantities when conditioned on their values at initialisation (possibly random).

We are interested in the infinite-width limit of large $n$. We aim at showing that, as $n\to\infty$, for each fixed input $x$ the network's output $F(x)$ converges to a multivariate normal, whose covariance matrix $Q(x)\in\R^q\times\R^q$ and mean vector $M(x)\in\R^q$ are deterministic functions of the hyper-parameters $\emme$ and $\esse$. In short, for any fixed input $x$, we want to establish that
\begin{equation}\label{eq:Gauss_F}
    F(x) \to \mathcal N(M(x), Q(x))\,.\footnote{Clearly, to be rigourous one needs to specify which kind of convergence is intended; see Propositions \ref{prop:init} and \ref{prop:lazytrain}.}
\end{equation}
Note that, for two different inputs $x$ and $x'$, $F(x)$ and $F(x')$ are independent, as we assume that the stochastic parameters of the model are re-sampled every time that a new input is provided.

As a remark, by taking the limit $n\to\infty$ we mean considering a sequence of distinct networks of increasing widths, all initialised and trained in the same way. To be rigorous, one ought to add explicit superscripts $^{(n)}$ to the various quantities to stress their dependence on the network's width. So, one should actually write $F^{(n)}$, and say that its mean and covariance $M^{(n)}$ and $Q^{(n)}$ can be expressed in terms of $\emme^{(n)}$ and $\esse^{(n)}$. What we will show is that, for each $x$, $F^{(n)}(x)\to F(x)\sim\mathcal N(M(x), Q(x))$, where $M$ and $Q$ are the limits of $M^{(n)}$ and $Q^{(n)}$. However, we believe that stressing this explicit dependence on $n$ would result in an excessively heavy notation. Therefore, we will always omit the superscript $^{(n)}$, and we will freely speak of ``infinite-width limit'' of a network, with the understanding that this has to be intended as the limit of a sequence of networks.
\subsection{Infinite-width limit}
We start by focusing on the hidden layer, which we denote as $Y^0$. Its nodes can be expressed as
$$Y^0_j(x) = \sum_{k=1}^p W^0_{jk} x_k =\frac{1}{\sqrt p}\sum_{k=1}^p \esse^0_{jk} \zeta^0_{jk} x_k +\frac{1}{\sqrt p}\sum_{k=1}^p \emme^0_{jk} x_k\,,$$
for any fixed input $x\in\R^p$. As the $\zeta^0_{jk}$'s are iid standard Gaussian random variables, we have that
$$Y^0(x)\sim\mathcal N(M^0(x), Q^0(x))\,.$$
This means that $Y^0$ is a $n$-dimensional multivariate normal, with mean vector and covariance matrix given by
$$ M^0_j(x) =  \frac{1}{\sqrt p}\sum_{k=1}^p \emme^0_{jk} x_k\,;\qquad\qquad Q^0_{jj'}(x) =\kron_{jj'}\frac{1}{p}\sum_{k=1}^p(\esse^0_{jk}x_k)^2\,.$$
As $Q^0(x)$ is diagonal, all the components of $Y^0(x)$ are independent, and we can actually write 
\begin{equation}\label{eq:hidden_nodes}
    Y^0_j(x) = \sqrt{Q^0_{jj}(x)}\,\bar \zeta^0_j + M^0_j(x)\,,
\end{equation}
where the $\bar \zeta^0_j$'s are independent standard normals. 

Now, define the random variable
$$\Phi^0_j(x) = \phi(Y^0_j(x))\,.$$
Clearly, we have $F_i(x) =\sum_{j=1}^n W^1_{ij}\Phi^0_j(x)$. Expanding the components of $W^1$ we can write
$$F_i(x) = \frac{1}{\sqrt n}\sum_{j=1}^n \esse_{ij}^1\zeta_{ij}^1\Phi^0_j(x) +\frac{1}{\sqrt n}\sum_{j=1}^n \emme_{ij}^1 \Phi_j^0(x)\,.$$

For any fixed input $x$, in the limit $n\to\infty$, we have an infinite sum of independent random variables, which are not identically distributed. In order to establish the convergence to a multivariate normal distribution, we need to control the variance and some higher moment of these variables, and hence require that the hyper-parameters have the correct order of magnitude. This is the case when the network is suitably initialised, and the result remains true during the training, as long as the hyper-parameters stay close enough to their initial values.

Note that for any finite width $n$, we can explicitly evaluate the network's mean $M$ and covariance $Q$. For the mean, we have
\begin{equation}\label{eq:M}
M_i(x) = \E[F_i(x)] = \frac{1}{\sqrt n}\sum_{j=1}^n\emme_{ij}^1\E[\Phi_j^0(x)]\,.
\end{equation}
As for $Q(x)$, we have $Q_{ii'}(x) = \C_{ii'}[F(x)] = \E[F_i(x) F_{i'}(x)] - \E[F_i(x)]\E[F_{i'}(x)]$, which becomes
\begin{equation}\label{eq:Q}
Q_{ii'}(x) = \kron_{ii'}\frac{1}{n}\sum_{j=1}^n(\esse^1_{ij})^2\E[\Phi^0_j(x)^2] + \frac{1}{n}\sum_{j=1}^n\emme^1_{ij}\emme^1_{i'j}\V[\Phi^0_j(x)]\,,
\end{equation}
where we used the fact that the nodes of the hidden layer are independent and so the covariance of $\Phi^0(x)$ is diagonal. Once we will have established that the limit of infinite width leads to a Gaussian output, its mean and covariance will be given by the limit $n\to\infty$ of the above expressions.

We now state some rigorous results. The next proposition (see Appendix \ref{app:proofs} for the proof) builds on a central limit theorem for triangular arrays, due to \cite{bentkus:lyap_bound}.
\begin{proposition}\label{prop:bentkus}
For any fixed input $x$ and width $n$, define $M(x)$ and $Q(x)$ as in \eqref{eq:M} and \eqref{eq:Q}. Let $Z(x)\sim\mathcal N(M(x), Q(x))$ and denote as $\Cc$ the class of measurable convex subsets of $\R^q$. Let $F$ be defined as in \eqref{eq:net}. Then
$$\sup_{C\in\Cc}|\Pp(F(x)\in C) - \Pp(Z(x)\in C)| \leq \kappa q^{1/4} \frac{B(\emme, \esse)}{\sqrt n}\,,$$
where $\kappa< 4$ is an absolute constant and
$$B(\emme,\esse)\leq q^{1/2}\frac{\frac{1}{n}\sum_{j=1}^n\sum_{i=1}^q (2|\esse^1_{ij}|^3 + 8|\emme^1_{ij}|^3)\E[|\Phi^0_j(x)|^3]}{\left(\frac{1}{n}\sum_{j=1}^n\E[\Phi_j^0(x)^2]\min_{i=1\dots q}(\esse_{ij}^1)^2\right)^{3/2}}\,.$$
In particular, if $B(\emme,\esse)= o(\sqrt n)$ for $n\to\infty$, then $F(x)- Z(x)\to 0$, in distribution.
\end{proposition}
As a corollary of the above result, if the stochastic network acts as a classifier, its performance is related to the one of its Gaussian approximation.
\begin{corollary}\label{cor:loss_bound}
Assume that the network deals with a classification problem, where for each instance $x$ there is a single correct label $y=f(x)\in\{1\dots q\}$. With the notation of Proposition \ref{prop:bentkus}, for each fixed input $x\neq 0$, define as $\hat f(x) = \argmax_{i=1\dots q}F_i(x)$ and $\bar f(x) = \argmax_{i=1\dots q}Z_i(x)$. We have
$$|\Pp(\hat f(x) = f(x)) - \Pp(\bar f(x) = f(x))| \leq \kappa q^{1/4} \frac{B(\emme,\esse)}{\sqrt n}\,.$$
\end{corollary}
\begin{proof}
For each $k=\{1\dots q\}$, the set $\{z\in\R^q:z_k>\max_{i\neq k}z_i\}$ is convex. Hence the claim directly follows from Proposition \ref{prop:bentkus}.
\end{proof}
\subsection{Initialisation and lazy training}
With a suitable random initialisation of the hyper-parameters, and in a lazy training regime, we show that, as $n\to\infty$, our stochastic network has a Gaussian limit, in the sense that the quantity $B/\sqrt n$ of Proposition \ref{prop:bentkus} vanishes as $n\to\infty$. For simplicity, we shall assume that the activation function $\phi:\R\to\R$ is Lipshitz continuous (although we do not need to specify the Lipschitz constant).

We let all the network hyper-parameters be independently initialised in the following way:
\begin{equation}\label{eq:init}
\begin{aligned}
    &\emme_{jk}^0\sim\mathcal N(0, 1)\,;\qquad&&\emme_{ij}^1\sim\mathcal N(0, 1)\,;\\
    &\esse_{jk}^0=1\,;\qquad&&\esse_{ij}^1=1\,,
\end{aligned}
\end{equation}
For convenience we write $\hat\Pp$ for the probability measure representing the above initialisation, while $\Pp$ is the probability measure describing the intrinsic stochasticity of the network. These two sources of randomness are always assumed to be independent.
\begin{proposition}[Initialisation]\label{prop:init}
Consider a sequence of networks of increasing width initialised according to \eqref{eq:init}, and whose activation function $\phi$ is Lipshitz continuous. For any fixed input $x\neq 0$, defining $B$ as in Proposition \ref{prop:bentkus}, we have $\frac{B(\emme,\esse)}{\sqrt n}\to 0$, as $n\to\infty$, in probability with respect to the random initialisation $\hat\Pp$. More precisely, $B(\emme,\esse) = O(1)$ wrt $\hat\Pp$, as $n\to\infty$. In particular, at the initialisation the network tends to a Gaussian limit, in distribution wrt the intrinsic stochasticity $\Pp$ and in probability wrt $\hat\Pp$.
\end{proposition}
\begin{proofsketch}
The proof is deferred to Appendix \ref{app:proofs}. The main idea is that, since all the hyper-parameters are independent under \eqref{eq:init}, the standard central limit theorem yields that the upper-bound for $B$ stated in Proposition \ref{prop:bentkus} tends to a finite limit as $n\to\infty$.
\end{proofsketch}
%\todo[inline]{maybe introduce notation,say $\mu, \nu$ so that when you make statements in probability you can say wrt $\mu$, $\nu$ or $\mu\otimes \nu$.}
The next proposition states that the limit will still be valid as long as the hyper-parameters do not move too much from their initialisation (lazy training). 
\begin{proposition}[Lazy training]\label{prop:lazytrain}
Fix a constant $J>0$ independent of $n$, and assume that $\phi$ is Lipshitz. For a network of width $n$, with initial configuration $(\widetilde\emme,\widetilde\esse)$ drawn according to $\hat\Pp$ as in \eqref{eq:init}, denote as $\Bb_J$ the ball
$$
\Bb_J = \left\{(\emme,\esse)\;:\quad\|\emme^0-\widetilde\emme^0\|^2_{F,2} + \|\emme^1-\widetilde\emme^1\|_{F,2}^2+\|\esse^0-\widetilde\esse^0\|^2_{F,2} + \|\esse^1-\widetilde\esse^1\|_{F,2}^2\leq J^2\right\}\,,
$$
where $\|\cdot\|_{F,2}$ denotes the 2-Frobenius norm of a matrix.
Let $B$ be defined as in Proposition \ref{prop:bentkus}. For any fixed input $x\neq 0$ we have $B(\emme, \esse)= O(1)$ as $n\to\infty$, uniformly on $\Bb_J$, in probability with respect to the random initialisation $\hat\Pp$.
\end{proposition}
\begin{proofsketch}
The proof is rather long and technical, and is deferred to Appendix \ref{app:proofs}. However, the idea is simple and consists in showing that, under the lazy training assumption $(\emme,\esse)\in\Bb_J$, $B$ undergoes a change of order $O(1)$ during the training. Since by Proposition \ref{prop:init} we know that $B$ is of order $O(1)$ at the initialisation, we can conclude.
\end{proofsketch}

In the next section, we will see that the lazy training constraint can be restated in terms of a bound on the Kullback-Leibler divergence between the initial and final distributions of the stochastic parameters. This fact will allow us to ensure that the constraint is satisfied when training the network to optimise a PAC-Bayesian objective.

%As a final remark, it can be easily checked that the lazy training constraint $\|\esse^0-\widetilde\esse^0\|^2_{F,2} + \|\esse^1-\widetilde\esse^1\|_{F,2}^2\leq J^2$ prevents the change in $\esse$ to have an impact on the final distribution of the network, as it induces a variation vanishing with $n$ in the mean vector $M$ and covariance matrix $Q$ of the output distribution. On the other hand, the constraint $\|\emme^0-\widetilde\emme^0\|^2_{F,2} + \|\emme^1-\widetilde\emme^1\|_{F,2}^2\leq J^2$ still allows the change in $\emme$ to bring a variation of order $O(1)$ in $M$ and $Q$. 
%\todo[inline]{I think I see why but not super clear. Maybe explain a bit better? Is it because the $s$'s affect the distribution through their $l^2$ whereas the $m$ can vary by $1/\sqrt{n}$?}
%For this reason, in our experiments, we chose to train the means $\emme$'s only, keeping the standard deviations fixed to their initial value $1$.
\section{PAC-Bayesian framework}
Consider a standard classification problem, where to each instance $x\in\X\subseteq\R^p$ corresponds a unique correct label $y=f(x)\in\Y=\{1\dots q\}$. The goal is to build an algorithm that is able to find a good prediction of $y$ given $x$. We assume that the $x$'s are distributed according to some probability measure $\Pp_X$ on $\X$. To train our algorithm, we have access to a sample $S = (X_h)_{h=1\dots m}$, which is correctly labelled (for every $X_h\in S$ we know $f(X_h)$). Each $X_h$ is an independent draw from $\Pp_X$, so that $\Pp_S = \Pp_X^{\otimes m}$. We let $\ell$ be the 01-loss:
$$\ell(\hat y,y)=\begin{cases} 0&\text{if $\hat y = y$;}\\1 &\text{otherwise.}\end{cases}$$

%We let $\hat f_w(x)$ denotes the prediction for the instance $x$ for a network with parameters $w$. The empirical loss $L_S(w)$ is defined as the average of the 01-loss on the training set, $L_S(w) = \frac{1}{m}\sum_{x\in S}\ell(\hat f_w(x), f(x))$, while the true loss is $L_X(w)=\E_X[\ell(\hat f_w(X), f(X))]$.
We let $\hat f_w(x)$ denote the prediction for the instance $x$, for a network with parameter configuration $w$. The empirical loss $L_S(w)=\frac{1}{m}\sum_{x\in S}\ell(\hat f_w(x), f(x))$ is the average of the 01-loss on the training set, while the true loss is $L_X(w) = \E_X[\ell(\hat f_w(X), f(X))]$.

The PAC-Bayesian framework \citep{mcall-pac-bayesian, mcallester, catoniPAC, guedj2019primer, alquier2021user} deals with stochastic neural classifiers. We consider a prior probability measure $\PR$ on the random parameters, which has to be chosen independently of the specific realisation of the random dataset $S$ used for the training. After the training, the parameters will be described by a new probability measure $\Q$ (the so-called posterior), clearly $S$-dependent. The idea is that if $\PR$ and $\Q$ are not too ``far'' from each other, then the network will generalise well. 

Under the posterior, we define the expected true loss $L_X(\Q) = \E_{W\sim \Q}[L_X(W)]$ and the expected empirical loss $L_S(\Q) = \E_{W\sim \Q}[L_S(W)]$. The PAC-Bayesian bounds are upper bounds on $L_X(\Q)$, which hold with high probability on the random draw of the training set $S$. They usually involve the expected empirical error $L_S(\Q)$ and a divergence term in the form of the Kullback-Leibler divergence between $\Q$ and $\PR$: $\KL(\Q\|\PR) = \E_\Q[\log(\mathrm{d}\PR/\mathrm{d}\Q)]$. We will use the following result, due to \cite{langford_bounds} and \cite{maurer}. 

\begin{proposition}\label{prop:bound}
Fix a data-independent prior $\PR$. With probability higher than $1-\delta$ on the choice of the training set $S=(X_h)_{h=1\dots m}$\footnote{Here we assume that the training set $S$ has size $m\geq 8$.},
\begin{equation}\label{eq:bound}
    L_X(\Q)\leq \kl^{-1}\left(L_S(\Q)\bigg|\frac{\KL(\Q\|\PR)+\log\frac{2\sqrt m}{\delta}}{m}\right)\,,
\end{equation}
for any posterior $\Q$. Here, we have defined $\kl^{-1}(u|c) = \sup\{v\in[0,1]:\kl(u\|v)\leq c\}$, where $\kl(u\|v)=u\log\frac{u}{v}+ (1-u) \log\frac{1-u}{1-v}$.
\end{proposition}
We can hence devise the following training algorithm \citep{mcall-pac-bayesian}:
\begin{itemize}
    \item Fix $\delta\in(0,1)$ and a prior $\PR$ for the network stochastic parameters;
    \item Collect a sample $S$ of $m$ iid datapoints;
    \item Compute the optimal posterior $\Q$ minimising \eqref{eq:bound};
    \item Implement a stochastic network characterised by the law $\Q$.
\end{itemize}
In practice, in essentially any realistic scenario the algorithm above cannot be implemented. Hence, one has to simplify the problem requiring that $\PR$ and $\Q$ belong to some simple class of distributions. 
\subsection{PAC-Bayesian training}
Following the approach of \cite{dziugaite2017computing}, we assume that both $\PR$ and $\Q$ are multivariate normal distributions with diagonal covariance matrices. In other words, the random parameters of the network are independent normal random variables. For the posterior, $\emme$ and $\esse$ denote the $N$-dimensional vectors of the means and the standard deviations, while $\widetilde\emme$ and $\widetilde\esse$ refer to the prior. In short, $\PR = \mathcal N(\widetilde\emme, \diag(\widetilde\esse^2))$ and $\Q = \mathcal N(\emme, \diag(\esse^2))$. In this Gaussian setting, $\KL(\Q\|\PR)$ takes a simple form:
\begin{equation}\label{eq:KL_gauss}
    \KL(\Q\|\PR) = \frac{1}{2}\left(\sum_\alpha\left(\frac{\esse_\alpha}{\widetilde\esse_\alpha}\right)^2-N+\sum_\alpha\left(\frac{\emme_\alpha-\widetilde\emme_\alpha}{\widetilde\esse_\alpha}\right)^2 + 2\sum_\alpha\log\frac{\widetilde\esse_\alpha}{\esse_\alpha}\right)\,,
\end{equation}
where the index $\alpha$ runs over all the components of the hyper-parameters.

Now, the most troublesome term in \eqref{eq:bound} is  $L_S(\Q)$, which in general cannot be computed explicitly. However, we can obtain a Monte Carlo (MC) estimate $\hat L_S(\Q)$ of this quantity, by sampling a few realisations of the parameters from $\Q$.

Now, the idea is to perform a gradient descent (GD) optimisation on the PAC-Bayesian objective \citep{dziugaite2017computing, perezortiz2021tighter}. Note that \eqref{eq:KL_gauss} is differentiable with respect to $\emme$ and $\esse$ (which are the trainable hyper-parameters of the posterior). However, $\hat L_S(\Q)$ has a null gradient almost everywhere, as this is the case for $L_S(w)$ for each realisation $w$ used in the estimate. The standard way to overcome this issue is to use a surrogate of the 01-loss for the training, such as some variant of the cross-entropy \citep{dziugaite2017computing, perezortiz2021tighter}. Notably, although $\hat L_S(\Q)$ has a null gradient, this is not the case for $L_S(\Q)$ (see Section \ref{sec:01training} and Figure \ref{fig:stairs} for more details). Hence, if we know exactly the output's distribution of the stochastic network, we might be able to use the 01-loss directly without the need of any surrogate. This is indeed the case for the Gaussian limit, as we will see in the next section. In a similar spirit, \cite{alquier-ridgway-chopin} studied the training of a linear binary classifier with Gaussian parameters. 

It is worth mentioning that similar considerations hold when using an almost everywhere constant activation function to train a stochastic network. In this regard, \cite{pacbayeslinclass, letarte2020dichotomize, biggs2020differentiable} developed an interesting variant of PAC-Bayesian training for binary classifiers with the sign activation function ($\phi = \text{sign}$). In that setting, the simple form of the output of each layer allows for a more explicit expression of the distribution of the hidden nodes, which permits overcoming the fact that the binary activation function is non-differentiable. 
\section{PAC-Bayesian training in the Gaussian limit}\label{sec:PACGauss}
Instead of doing the standard PAC-Bayesian training with a surrogate loss, we can train our wide stochastic network by assuming that its Gaussian approximation is exact. However, once completed the training, we will need to evaluate the final bound without such an assumption.

At the initialisation, for a network initialised according to $\hat\Pp$ as in \eqref{eq:init}, the Gaussian approximation is asymptotically exact for large $n$. Moreover, the following lemma ensures that controlling the $\KL$ divergence is enough to claim that the network is in the lazy training regime of Proposition \ref{prop:lazytrain}. Hence, a wide stochastic network is asymptotically Gaussian throughout its PAC-Bayesian training.

\begin{lemma}
Define the multivariate Gaussian distributions $\PR = \mathcal N(\widetilde \emme, \diag(\widetilde \esse^2)) = \mathcal N(\widetilde \emme, \Id)$ and $\Q = \mathcal N(\emme, \diag(\esse^2))$ for the parameters of a stochastic network. We have
$$\|\emme^0-\widetilde\emme^0\|^2_{F,2} + \|\emme^1-\widetilde\emme^1\|_{F,2}^2+\|\esse^0-\widetilde\esse^0\|^2_{F,2} + \|\esse^1-\widetilde\esse^1\|_{F,2}^2\leq 2\KL(\Q\|\PR)\,.$$
\end{lemma}
\begin{proof}
From \eqref{eq:KL_gauss}, we conclude noticing that $u^2-1-2\log u\geq (u-1)^2$, for all $u> 0$.
\end{proof}

The rest of this section is organised as follows. First, we show that it is possible to get a non-zero gradient from the expected 01-loss in the Gaussian limit. Then, we discuss how to evaluate the gradients of the output's mean and covariance with respect to the hyper-parameters. Finally, we deal with how to obtain a rigorous PAC-Bayesian bound after the training.
\subsection{Training with the 01-loss}\label{sec:01training}
\begin{wrapfigure}{r}{0.35\textwidth}
\includegraphics[width=0.35\textwidth]{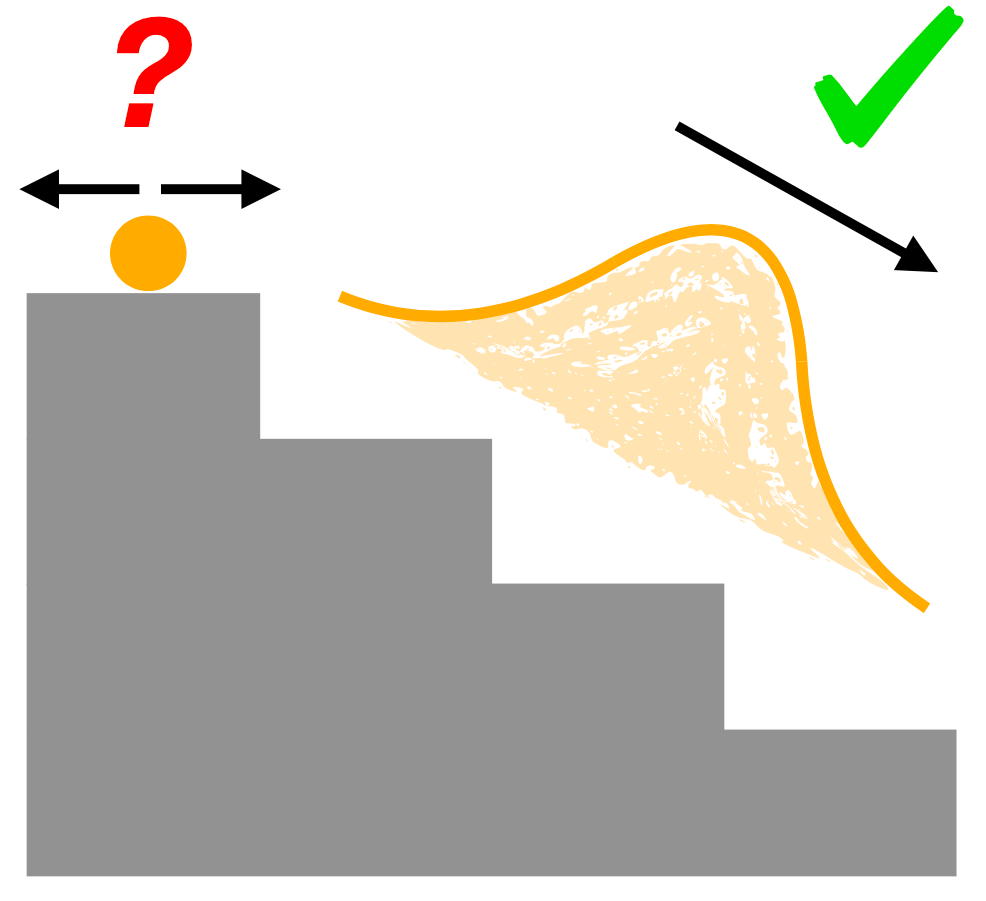}
\caption{When ``going down the stairs'' via GD, each single realisation lies on a horizontal step and has an uninformative null gradient, but the whole distribution has a global view of the stairs and can find the good direction.\\\vspace{-20pt}}
\label{fig:stairs}
\end{wrapfigure}
For $x\in\X\subseteq\R^p$, we want to find the correct $y=f(x)$ among $q$ possible labels $i = 1,\dots, q$. We consider a Gaussian network with output $F(x) \sim \mathcal N(M(x), Q(x))$, whose random prediction is $\hat f(x) = \argmax_{i=1\dots q}F_i(x)$. Denoting as $\ell$ the 01-loss, it is natural to aim at minimising $\E[\ell(\hat f(x), f(x))]$ (where the expectation is over the stochastic parameters), since this quantity is actually equal to the probability of making a mistake for $x$: $\Pp(\hat f(x)\neq f(x))$. As we want to tackle the problem by performing gradient descent optimisation, if we assume that we are able to differentiate $M(x)$ and $Q(x)$ with respect to the network trainable hyper-parameters, all we need is to evaluate $\nabla_M\E[\ell(\hat f(x), y)]$ and $\nabla_Q\E[\ell(\hat f(x), y)]$.

Note that $\ell(\hat f(x), y)$ has a null gradient almost everywhere for any random realisation of the network. For this reason, a non-stochastic network cannot be trained directly with the 01-loss. However, this is not the case for a stochastic network. The reason for that can be intuitively explained by thinking of what happens if we try to ``go down the stairs'' with GD, as illustrated in Figure \ref{fig:stairs}. A single realisation of the network will be a point on a horizontal step: there is no way to understand the right direction in order to go down. However, if we consider the whole stochastic distribution of the network, it spreads over all the steps, and it has a global view of the stairs. It is hence not surprising that the gradient of the expected loss is non-zero.

For binary classification tasks, the expected 01-loss reads $\E[\ell(\hat f(x),1)] = \Pp(F_2(x)>F_1(x))$ and $\E[\ell(F(x),2)] = \Pp(F_1(x)>F_2(x))$. %\,.$$
These quantities can be computed exactly:\footnote{A similar result was already derived in \cite{alquier-ridgway-chopin} for a simpler linear classifier.}
\begin{align*}
    &\E[\ell(\hat f(x),1)] = \Pp_{\zeta\sim\mathcal N(0,1)}\left(\zeta>\frac{M_1(x)-M_2(x)}{\sqrt{Q_{11}(x) + Q_{22}(x) - 2Q_{12}(x)}}\right)\,;\\
    &\E[\ell(\hat f(x),2)] = \Pp_{\zeta\sim\mathcal N(0,1)}\left(\zeta>\frac{M_2(x)-M_1(x)}{\sqrt{Q_{11}(x) + Q_{22}(x) - 2Q_{12}(x)}}\right)\,.
\end{align*}
Clearly, the two expressions above can be written explicitly in terms of the error function erf, as $\zeta$ is distributed as a standard normal and $\Pp(\zeta>u) = \frac{1}{2}(1-\text{erf}\,(u/\sqrt 2))$. It is then straightforward to see that $\E[\ell(\hat f(x), y)]$ is differentiable with respect to $M$ and $Q$, with non-zero derivatives.

When there are more than two classes, things become more complicated. It is however possible to exploit the Gaussianity and obtain a MC estimator of the expected loss, whose gradient with respect to $M$ and $Q$ is computable and not trivially zero. We refer to Appendix \ref{app:multiclass} for  details.
                                        
\subsection{Derivatives of \texorpdfstring{$M$}{M} and \texorpdfstring{$Q$}{Q}}
We have so far established that we can effectively differentiate the expected loss with respect to $M$ and $Q$. Still, to train the network we will need to evaluate the gradients with respect to the hyper-parameters $\emme$ and $\esse$. Now, recall that $\Phi^0_k(x)$ is in the form $\phi(a \zeta + b)$, with $a = \sqrt {Q^0_{kk}(x)}$, $b= M^0_k(x)$, and $\zeta\sim\mathcal N(0,1)$. When the activation function $\phi$ is simple enough, $\E[\phi(a \zeta+b)]$ and $\E[\phi(a \zeta+b)^2]$ have closed-form expressions. Exploiting this fact, it is possible to evaluate the $\emme^0$- and  $\esse^0$-derivatives of $M$ and $Q$, needed in order to train the network with gradient-based methods. This if for instance the case for $\phi=\ReLU$ and $\phi=\sin$ (see Appendix \ref{app:relusin}).

\subsection{Final computation of the bound}
Once completed the training, we need to abandon the Gaussian approximation to compute the final bound. We will follow the same approach as \cite{dziugaite2017computing} and \cite{perezortiz2021tighter}.

Let $W_1,\dots, W_N$ be $N$ independent realisations of the whole set of network stochastic parameters, drawn according to $\Q$. For $\delta'\in(0,1)$, with probability at least $1-\delta'$ \citep{langfordcaruana}
\begin{equation}\label{eq:emp_err_est}
L_S(\Q) \leq \kl^{-1}\left(\hat L_S(\Q)\big|\tfrac{1}{N}\log\tfrac{2}{\delta'}\right)\,,
\end{equation}
where $\kl^{-1}$ is defined in Proposition \ref{prop:bound} and we have defined $\hat L_S(\Q) = \frac{1}{N}\sum_{h=1}^N L_S(W_h)$. Since $\kl^{-1}$ is increasing in its first argument%,,
, Proposition \ref{prop:bound} yields that with probability at least $1-\delta-\delta'$
\begin{equation}\label{eq:final_bound}
L_X(\Q) \leq \kl^{-1}\left(\kl^{-1}\left(\hat L_S(\Q)\big|\tfrac{1}{N}\log\tfrac{2}{\delta'}\right)\bigg|\frac{\KL(\Q\|\PR)+\log\frac{2\sqrt m}{\delta}}{m}\right)\,.
\end{equation}
This method is often computationally very expensive, especially for large values of $N$. However, using a standard re-parameterisation trick from \cite{kingma15} helps to speed-up the evaluation, as it makes possible to obtain a realisation of the network by sampling only $d+n$ standard normals, instead of all the $p\times n^2\times q$ stochastic parameters.

As a final remark, an alternative way to get an exact result from the Gaussian approximation is to use an upper bound, such as the one in Corollary \ref{cor:loss_bound}, to control the finite-size correction to the expected empirical loss. However, for networks with $O(10^3)$ hidden nodes, like those that we used in our experiments, this last approach gives looser bounds compared to the method described above, at least when the number $N$ of samples used for the MC estimate $\hat L_S(\Q)$ is of order $O(10^5)$.
\section{Experimental results}
In this section, we present some empirical results to validate our theoretical findings. First, we compare the Gaussian predictions with the distribution of the output nodes of a wide stochastic network. Then, we report the results obtained by training a stochastic network on MNIST, and on a binary version of it, with our Gaussian method and with standard PAC-Bayesian procedures like those from \cite{dziugaite2017computing} and \cite{perezortiz2021tighter}. On both datasets, the Gaussian method led to tighter final generalisation bounds. The PyTorch code developped for this paper is available at \url{https://github.com/eclerico/WideStochNet}. For the sake of conciseness, we refer to Appendix \ref{app:exp} for an exhaustive account of the experimental details.

In order to keep the experimental setting as simple as possible, we opted for training only the means $\emme$ (keeping the standard deviations $\esse$ fixed at their initial value), similarly to what was done in \cite{letarte2020dichotomize}. Moreover, coherently with the rest of this paper, all the networks that we used
\begin{wrapfigure}{l}{0.6\linewidth}
\includegraphics[width=0.6\textwidth]{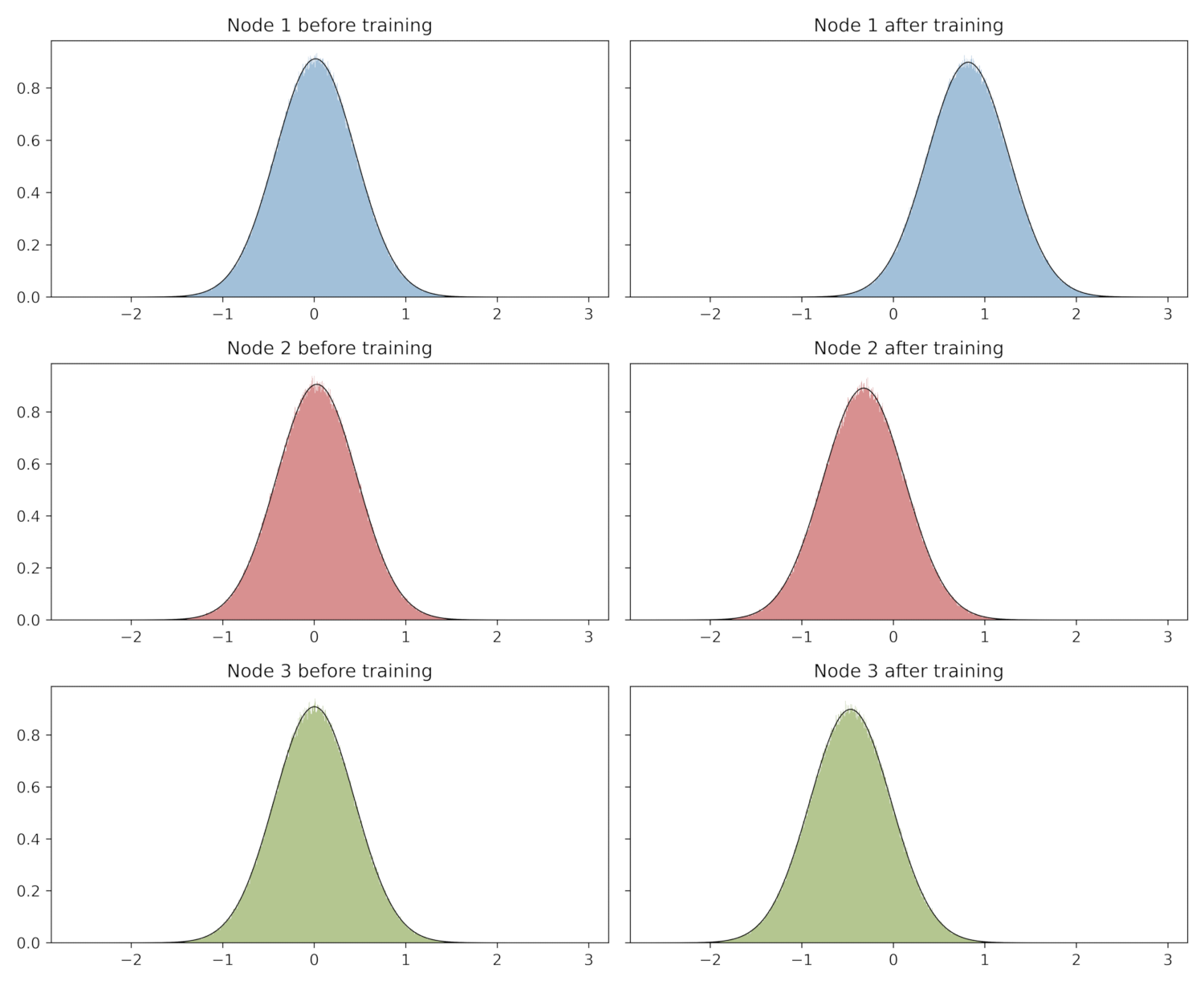}
\caption{Distributions of the three output nodes of a wide stochastic network trained on a toy classification task. In black the theoretical predictions.}
\label{fig:toycomp}
\end{wrapfigure} 
had no bias. The PAC-Bayesian priors were chosen in a completely data-independent fashion, and coincided with the distribution of the network at initialisation, as suggested by \cite{dziugaite2017computing}.

We start by considering a toy dataset, whose datapoints were sampled from three multivariate standard normal distributions (labelled as $1$, $2$, $3$) in $\mathbb R^4$, and then projected on the unit sphere in $\mathbb R^4$. A stochastic network with one hidden layer of $n=1200$ nodes was trained to predict from which of the three Gaussian clusters each point comes. The histograms in Figure \ref{fig:toycomp} represent the distributions of the network's output nodes, both before and after the training. They have been obtained for a single example by sampling $10^6$ realisations of the random parameters. The theoretical predictions of the Gaussian profiles are plotted in black. The agreement with the histograms is striking, showing that the network is essentially Gaussian already for $O(10^3)$ hidden nodes.

We now focus on the experiments on a binary version of the MNIST dataset, where the training dataset consisted of $m=60000$ images. We considered a stochastic network with $n=1200$ hidden nodes and ReLU activation function, initialised as in \eqref{eq:init}. We tried four training methods, based on different training objectives. The three ``standard'' PAC-Bayesian procedures used the objectives
\begin{align}
\begin{split}\label{eq:objs}
    &\mathtt{McAll} = \bar L_S(\Q) + \sqrt{\frac{\KL(\Q\|\PR) + \log\frac{2\sqrt m}{\delta}}{2m}}\,;\\
    &\mathtt{lbd} = \frac{\bar L_S(\Q)}{(1-\lambda/2)} + \frac{\KL(\Q\|\PR) + \log\frac{2\sqrt m}{\delta}}{m\lambda(1-\lambda/2)}\,;\\
    &\mathtt{quad} = \left(\sqrt{\bar L_S(\Q) + \frac{\KL(\Q\|\PR) + \log\frac{2\sqrt m}{\delta}}{2m}}+ \sqrt{\frac{\KL(\Q\|\PR) + \log\frac{2\sqrt m}{\delta}}{2m}}\right)^2\,,
\end{split}
\end{align}
where $\bar L_S(\Q)$ is the expectation under $\Q$ of the empirical cross-entropy loss divided by $\log 2$. The objective $\mathtt{McAll}$ is from \cite{dziugaite2017computing}, while $\mathtt{quad}$ comes from \cite{perezortiz2021tighter} and $\mathtt{lbd}$ was originally derived by \cite{pmlr-v76-thiemann17a} and later used by \cite{perezortiz2021tighter}. In $\mathtt{lbd}$, $\lambda\in(0, 1)$ is also a trainable parameter.  

As we are dealing with binary classification, for the ``Gaussian'' method (described Section \ref{sec:PACGauss}), the expected value $L_S(\Q)$ of the 01-loss can be evaluated directly (see Section \ref{sec:01training}). We could hence directly optimise \eqref{eq:bound}, using the objective
\begin{equation}\label{eq:objbound}
\mathtt{invkl} = \kl^{-1}\left(L_S(\Q)\bigg|\frac{\KL(\Q\|\PR)+\log\frac{2\sqrt m}{\delta}}{m}\right)\,.
\end{equation}

Table \ref{tab:binmnist} illustrates the results of the experiment. The column ``Bound'' reports the values of the PAC-Bayesian bound \eqref{eq:final_bound}. For the upper bound \eqref{eq:emp_err_est} on the empirical error, we used $N=150000$ independent realisations of the net, $\delta'=0.01$, and $\delta=0.025$, so that the final generalisation bounds
\begin{wraptable}{r}{.6\textwidth}
\begin{minipage}{\linewidth}
\centering
\caption{Binary MNIST}
\label{tab:binmnist}
\resizebox{\textwidth}{!}{\begin{tabular}{l|cc|ccc}
    \toprule
    Method      & Bound  & Test error        & G Bound         & G Loss        & Penalty        \\
    \midrule
    $\mathtt{invkl}$& \textbf{.1773}& .0694\textsubscript{\textpm.0040} &.1741  & .0676& .0492          \\
    \midrule
    $\mathtt{McAll}$& .1978         & \textbf{.0456}\textsubscript{\textpm.0025}& .1947           & .0428         & .1006 \\
    $\mathtt{lbd}$& .1856         & .0543\textsubscript{\textpm.0030}& .1825           & .0520         & .0752          \\
    $\mathtt{quad}$& .1855         & .0533\textsubscript{\textpm.0030}& .1823           & .0515         & .0757          \\
    \bottomrule
\end{tabular}}
\bigskip
\bigskip
\end{minipage}\\
\begin{minipage}{\linewidth}
\centering
\caption{MNIST}
\label{tab:mnist}
\resizebox{\textwidth}{!}{\begin{tabular}{l|cc|ccc}
    \toprule
    Method      & Bound  & Test Error        & G Bound         & G Loss        & Penalty        \\
    \midrule
    $\mathtt{invkl}$& \textbf{.2807}& \textbf{.1083}\textsubscript{\textpm.0039} &.2773  & .1114& .0821          \\
    \midrule
    $\mathtt{McAll}$& .4158         & .3189\textsubscript{\textpm.0097}& .4120           & .3265         & .0155 \\
    $\mathtt{lbd}$& .3736         & .2639\textsubscript{\textpm.0085}& .3699           & .2717         & .0216          \\
    $\mathtt{quad}$& .3735         & .2637\textsubscript{\textpm.0083}& .3698           & .2716         & .0217          \\
    \bottomrule
\end{tabular}}
\end{minipage}
\end{wraptable} 
hold with probability higher than $0.965$ on the random selection of the training set. The colum ``Test Error'' reports the average test error on a held-out dataset and its standard deviation. These values were evaluated on 10000 independent realisations of the test error. The two next columns refer to quantities computed within the Gaussian approximation: ``G Bound'' is the bound given by \eqref{eq:bound} and ``G Loss'' is the expected 01-loss. ``Penalty'' is the quantity $(\KL(\Q\|\PR)+\log\frac{2\sqrt m}{\delta})/m$. 

The ``Gaussian'' method yielded a tighter final bound than the ``standard'' ones. Yet, the best test error is achieved by $\mathtt{McAll}$. It is worth noting that the final bound obtained with $\mathtt{McAll}$ is slightly worse than the one from \cite{dziugaite2017computing}, where for a similar network of $1200$ hidden nodes a final bound of $.179$ was obtained, whilst our result is $.1921$. However, our setting is simpler: our network has no bias, the standard deviations are not trained, and there is no choice of the optimal prior among different initialisations.

Finally, we report the results of a similar experiment on the full MNIST dataset (with the original $10$ labels). The network is essentially the same one used for binary MNIST, with $1200$ hidden nodes and ReLU activation function. The main difference is that now we have $10$ output nodes. For the ``standard'' methods, we trained on the same objectives \eqref{eq:objs} as before, although this time we used a bounded version of the cross-entropy loss, as in \cite{perezortiz2021tighter}. $\bar L_S(\Q)$ is the expected value under $\Q$ of this bounded cross-entropy, averaged on the training set. The ``Gaussian'' method used the objective \eqref{eq:objbound}, where $L_S(\Q)$ is again the expected empirical 01-loss. Actually, as we were dealing with more than two classes, %multiclass classification, %we do not have a simple expression for the 01-loss, but we can proceed as described in Appendix \ref{app:multiclass}.
we could not exactly compute the expected 01-loss, since we do not have a simple closed-form expression for it, and we proceeded as described in Appendix \ref{app:multiclass}.

Table \ref{tab:mnist} reports the results of the experiment on the full MNIST dataset, where for the estimate of the final bounds we again used $N=150000$, $\delta'=0.01$, and $\delta=0.025$. Once more, the ``Gaussian'' method obtained a tighter result, with almost a $0.1$-gap with the bounds achieved by the other procedures. This time, the ``Gaussian'' method also attained the tightest test error. It is worth noticing that the PAC-Bayesian penalties of the standard methods are much lower than the respective losses\footnote{Training with longer time did not bring any relevant improvement, as the GD descent appeared to have already stabilised.}, something that did not occur in Table \ref{tab:binmnist}. We conjecture that this behaviour is due to the different rescaling of the cross-entropy loss. On the other hand, this is not the case for the Gaussian method, as the loss does not require any rescaling.
\section{Conclusions and perspectives}
In the present work, we derive a Gaussian limit for a simple one-layer stochastic architecture, and point out how this result can be used in practice for the PAC-Bayesian training of wide shallow networks. First, we rigorously prove the validity of the limit at the initialisation and in a lazy training regime. Then, we show empirically that the proposed training method can outperform some standard PAC-Bayesian training procedures. 

A main limitation of our approach is that it is limited to  shallow  networks with a single hidden layer. Indeed, our approach to establish the Gaussian limit relies on the fact that the hidden nodes are independent. This is not true anymore for any subsequent layer, and hence the CLT result that we use is no longer applicable. It is however worth mentioning that all the covariance matrices of the hidden layers are almost diagonal at the initialisation (as it is easy to check that the non-diagonal elements  scale as $1/\sqrt n$) and a lazy-training constraint equivalent to the one in Proposition \ref{prop:lazytrain} might be enough to help establishing a rigorous Gaussian limit holding for multilayer architectures. In any case, even if one were able to use a limit theorem holding for the sum of weakly dependent nodes, evaluating the output's law of the network would require the knowledge of the (non-diagonal) covariance matrices of the hidden layers.\footnote{Although the non diagonal elements are expected to scale as $1/\sqrt n$, the fact that they usually appear in sums of $O(n)$ terms can make their contribution non negligible. This was confirmed by a few empirical tests were we tried to only consider the diagonal elements of the covariance matrices of the central layers, and obtained inconsistencies  between the predicted and the empirical output laws.} As we are looking at wide networks, the storage of these matrices would require a considerable amount of computational memory. Nevertheless, it is still possible to exploit our Gaussian PAC-Bayesian training ideas for multilayer architectures. This was recently done by \cite{clerico22aistats}, which built on our work to obtain PAC-Bayesian bounds using the fact that the network's output is Gaussian when conditioned on the hidden layers. 
%(which would be enforced by training with a PAC-Bayesian objective) would constrain the hyper-parameters of all the central layers to be frozen at their initial values. Although this fact might help in showing rigorously that a Gaussian limit can be found also for multi-layers architectures, it does raise questions on whether having several wide layers would not just be detrimental for the PAC-Bayesian training. We refer to Appendix \ref{app:multilayer} for further comments. Nevertheless, it is still possible to exploit our Gaussian PAC-Bayesian training ideas for multilayer architectures. This was recently done by \cite{clerico22aistats}, which built on our work to obtain PAC-Bayesian bounds using the fact that the network's output is Gaussian when conditioned on the hidden layers. 

As a final remark, in the present work we did not treat the case of a network with biases. This is likely to be an elementary extension, which should not require much additional work. 
% Acknowledgments---Will not appear in anonymized version
%\acks{This work has been partially supported by the UK Engineering and Physical Sciences Research Council (EPSRC) through the grants EP/R513295/1 (DTP scheme) and CoSInES EP/R034710/1.}
\acks{EC is partially supported by the UK Engineering and Physical Sciences Research Council (EPSRC) through the grant EP/R513295/1 (DTP scheme) and AD by EPSRC CoSInES EP/R034710/1.  AD acknowledges support of the UK Defence Science and Technology Laboratory (DSTL) and EPSRC grant EP/R013616/1. This is part of the collaboration between US DOD, UK MOD and UK EPSRC under the Multidisciplinary University Research Initiative. The authors would like to thank Jian Qian for the valuable comments and suggestions.}
\newpage
\cleardoublepage
\phantomsection
\addcontentsline{toc}{section}{References}
\bibliography{bib}
\newpage
\appendix
\section{Omitted proofs}\label{app:proofs}
Throughout this section we use several notations for the norms of vectors and matrices. For $\gamma\geq 1$ and a vector $v$, $\|v\|_\gamma = (\sum_i |v_i|^\gamma)^{1/\gamma}$. If $A$ is a matrix, we define $\|A\|_{F,\gamma} = (\sum_{ij}|A_{ij}|^\gamma)^{1/\gamma}$ and $\|A\|_\gamma=\sup_{v:\|v\|_\gamma = 1}\|A v\|_\gamma$. We also recall that $\Pp$ denotes the intrinsic stochaticity of the network, while $\hat\Pp$ is the randomness due to the initialisation. These two sources of stochasticity are always supposed to be mutually independent. We denote as $\E$ the expectation wrt $\Pp$, and as $\hat\E$ the one wrt $\hat\Pp$. Moreover we write $\Gamma = O_{\hat\Pp}(n^\gamma)$ to mean that $\limsup_{n\to\infty}\frac{|\Gamma|}{n^\gamma}<\infty$ in probability wrt $\hat\Pp$, and $\Gamma = \Omega_{\hat\Pp}(n^\gamma)$ for $\limsup_{n\to\infty}\frac{|\Gamma|}{n^\gamma}>0$ in probability wrt $\hat\Pp$.

We want to prove a rigorous result of convergence to the Gaussian limit of wide stochastic networks. We will essentially make use of the next result, due to \cite{bentkus:lyap_bound}.
\begin{theorem}\label{thm:bentkus}
Let $X_1,\dots, X_n$ be independent random vectors in $\R^q$, such that $\E[X_j]=0$ for all $j$. Let $Y = \frac{1}{\sqrt n}\sum_{j=1}^nX_j$ and assume that the covariance matrix $\C[Y]$ is non singular. Let $Z\sim\mathcal N(0, \C[Y])$. Denote as $\frac{1}{\sqrt\C[Y]}$ the inverse of the positive square root of the matrix $\C[Y]$, and let $B_j = \E[\|\frac{1}{\sqrt\C[Y]}\,X_j\|_2^3]$ and $B = \frac{1}{n}\sum_{j=1}^n B_j$. Let $\mathcal C$ denote the class of all convex subsets of $\R^p$. Then, there exists an absolute positive constant $\kappa<4$ such that
\begin{equation}\label{eq:bentkus}
    \sup_{C\in\mathcal C}|\Pp(Y\in C) - \Pp(Z\in C)|\leq \kappa q^{1/4}\frac{B}{\sqrt n}\,. 
\end{equation}
\end{theorem}

Our goal is to prove a Gaussian limit as $n\to\infty$ for $F(x)$, whose components are given by 
$$F_i(x) = \frac{1}{\sqrt n}\sum_{j=1}^n \esse_{ij}^1\zeta_{ij}^1\Phi^0_j(x) +\frac{1}{\sqrt n}\sum_{j=1}^n \emme_{ij}^1 \Phi_j^0(x)\,.$$

Let us denote by $X_j$ the $q$-dimensional vector $X_j = (X_{1j}\dots X_{qj})$, with 
$$X_{ij} = \esse_{ij}^1 \zeta_{ij}^1\Phi^0_j(x) + \emme_{ij}^1(\Phi_j^0(x) - \E[\Phi_j^0(x)])\,.$$
Since all the $\zeta^1_{ij}$'s and the $\Phi^0_j$'s are independent, the $X_j$'s constitute a family of $n$ centred independent $q$-dimensional random vectors (wrt the intrinsic network stochasticity $\Pp$).

Clearly, we have $F(x) = \E[F(x)] + \frac{1}{\sqrt n}\sum_{j=1}^n X_j$. Let us define $Y = \frac{1}{\sqrt n}\sum_{j=1}^n X_j$. Note that, for all $x$, the covariance matrix $\C[Y]$ is given by $Q(x)$ (defined in \eqref{eq:Q}), no matter if $Y$ is Gaussian or not. Using the same notations of Theorem \ref{thm:bentkus}, assuming that $\C[Y]$ is non-singular, we have
$$\sup_{C\in\mathcal C}|\Pp(Y\in C) - \Pp(Z\in C)|\leq \kappa q^{1/4}\frac{B}{\sqrt n}\,.$$
To prove that the $Y$ behaves as a Gaussian for large $n$, we will show that $B= O(1)$ for large $n$. We can easily upperbound each $B_j$ as 
$$B_j\leq \frac{1}{\lambda[Y]^{3/2}}\E[\|X_j\|_2^3]\,,$$
where $\lambda[Y]>0$ is the smallest eigenvalue of $\C[Y]$. For simplicity, we have omitted the dependence of these quantities on $\emme$ and $\esse$. We will often do so throughout this section, in order to lighten the notation. 

Define $G = \frac{1}{n}\sum_{j=1}^n \E[\|X_j\|_2^3]$ and $\Lambda = \frac{1}{n}\sum_{j=1}^n\lambda[X_j]$. Clearly, as the $X_j$'s are independent, we have $\C[Y] = \frac{1}{n}\sum_{j=1}^n\C[X_j]$. In particular, we can easily find that $\lambda[Y]\geq \Lambda$.

We summarise what we have so far in the next lemma.
\begin{lemma}\label{lemma:Bbound}
With the notations introduced above, assuming that $\C[Y]$ is not singular, we have
$$B\leq\frac{1}{n}\sum_{j=1}^n \frac{1}{\Lambda^{3/2}}\E[\|X_j\|_2^3] = \frac{G}{\Lambda^{3/2}}\,.$$
\end{lemma}
%\todo[inline]{$\Lambda \neq 0$ although ineq still holds. If the $X_i$ have similar distribbutions this is fairly tight. If they differ a lot this may be very loose.}
Now, from  H\"older's inequality we have $\|X_j\|_2\leq\|\mathbf 1_q\|_6\|X_j\|_3= q^{1/6}\|X_j\|_3$, and so
\begin{equation}\label{eq:norm_bound}
\|X_j\|_2^3 \leq q^{1/2}\|X_j\|_3^3 = q^{1/2} \sum_{i=1}^q |X_{ij}|^3\,.
\end{equation}
Then, with some simple algebraic manipulation, and applying Jensen's inequality, we obtain 
$$\E[|X_{ij}|^3]\leq (|\esse_{ij}^1|^3\E[|\zeta_{ij}^1|^3]+8|\emme_{ij}^1|^3)\E[|\Phi_j^0(x)|^3]\,.$$
For convenience, we introduce the following notations
$$H_{ij} = (2|\esse_{ij}^1|^3+8|\emme_{ij}^1|^3)\E[|\Phi_j^0(x)|^3]\,;\qquad H_j =  \sum_{i=1}^q H_{ij}\,;\qquad H = \frac{1}{n}\sum_{j=1}^n H_j\,,$$
so that we have $G\leq q^{1/2}H$, since $\E[|\zeta|^3]=2\sqrt{2/\pi}<2$ for $\zeta\sim\mathcal N(0,1)$.

On the other hand, we can find a lowerbound for $\Lambda$ as well. Indeed, first we can notice that 
$$\C_{ii'}[X_j] = \delta_{ii'}(\esse^1_{ij})^2\E[\Phi^0_j(x)^2] + \emme^1_{ij}\emme^1_{i'j}\V[\Phi^0_j(x)]\,.$$
The first term is a diagonal matrix, while the second one is non-negative definite. Hence we can write
$$\lambda[X_j] \geq \E[\Phi^0_j(x)^2]\min_{i=1\dots q}(\esse^1_{ij})^2\,.$$

Defining 
\begin{equation}\label{eq:def_Theta}
\Theta = \frac{1}{n}\sum_{j=1}^n \E[\Phi^0_j(x)^2]\min_{i=1\dots q}(\esse^1_{ij})^2\leq\Lambda
\end{equation}
we have the following corollary of Lemma \ref{lemma:Bbound}.
\begin{corollary}\label{cor:app_benktus}
With the same notations as above, if $\C[Y]$ is non singular, we have
$$B \leq q^{1/2}\frac{H}{\Theta^{3/2}}\,.$$
\end{corollary}

Note that both $H$ and $\Theta$ can be evaluated explicitly, given the parameters of the networks, as long as $\phi$ allows for an explicit evaluation of $\E[|\Phi^0(x)|^\gamma]$, for $\gamma=1, 2, 3$. This means that we can give an exact upper bound to the finite-size error of the predicted 01-loss, for any configuration of the network.

We can now prove Proposition \ref{prop:bentkus} and Corollary \ref{cor:loss_bound} from the main text.

\begin{manualprop}{\ref{prop:bentkus}}
For any fixed input $x$ and width $n$, define $M(x)$ and $Q(x)$ as in \eqref{eq:M} and \eqref{eq:Q}. Let $Z(x)\sim\mathcal N(M(x), Q(x))$ and denote as $\Cc$ the class of measurable convex subsets of $\R^q$. Let $F$ be defined as in \eqref{eq:net}. Then
$$\sup_{C\in\Cc}|\Pp(F(x)\in C) - \Pp(Z(x)\in C)| \leq \kappa q^{1/4} \frac{B(\emme, \esse)}{\sqrt n}\,,$$
where $\kappa< 4$ is an absolute constant and
$$B(\emme,\esse)\leq q^{1/2}\frac{\frac{1}{n}\sum_{j=1}^n\sum_{i=1}^q (2|\esse^1_{ij}|^3 + 8|\emme^1_{ij}|^3)\E[|\Phi^0_j(x)|^3]}{\left(\frac{1}{n}\sum_{j=1}^n\E[\Phi_j^0(x)^2]\min_{i=1\dots q}(\esse_{ij}^1)^2\right)^{3/2}}\,.$$
In particular, if $B(\emme,\esse)= o(\sqrt n)$ for $n\to\infty$, then $F(x)- Z(x)\to 0$, in distribution.
\end{manualprop}
\begin{proof}
The result is a straight consequence of Theorem \ref{thm:bentkus} and Corollary \ref{cor:app_benktus}. Note that $\C[Y]$ is non singular as long as all the components of $\esse$ are non-zero, so as long as the bound in the statement is finite. 
\end{proof}

In the next section we show how, with a suitable random initialisation, we can assure that the network has an almost Gaussian behaviour. Successively, we will show that this behaviour is preserved during training, as long as the hyper-parameters do not move too much from their initial values. 
\subsection{Initialisation}
We consider the random initialisation:
\begin{equation*}\tag{\ref{eq:init}}
\begin{aligned}
    &\emme_{jk}^0\sim\mathcal N(0, 1)\,;\qquad&&\emme_{ij}^1\sim\mathcal N(0, 1)\,;\\
    &\esse_{jk}^0=1\,;\qquad&&\esse_{ij}^1=1\,,
\end{aligned}
\end{equation*}
As now we have two sources of randomness (the initialisation and the intrinsic stochasticity of the network) to avoid confusion we will denote as $\hat \E$, $\hat \Pp$, $\hat \V$ the expectations, probabilities and variances with respect to the initialisation, whilst $\E$, $\Pp$ and $\V$ refer to the network intrinsic stochasticity. 
\begin{lemma}\label{lemma:h_th}
Define $H$ and $\Theta$ as in the previous section for a network with parameters $(\emme,\esse)$ distributed according to $\hat\Pp$, as in \eqref{eq:init}. Assume that $\phi$ is Lipshitz continuous.
Then, for any fixed $x\neq 0$, $H\to h>0$ and $\Theta\to \theta>0$ in probability as $n\to\infty$, with respect to the random initialisation, where both $h$ and $\theta$ are finite.
\end{lemma}
\begin{proof}
First notice that, fixed an input $x\neq 0$ and fixed $n$, all the $\Phi^0_j$'s are iid, with respect to $\hat\Pp$, as all the components of $\emme^0$ and the $\esse_0$ are. As a consequence all the $H_j$'s are iid with respect to $\hat\Pp$ (note that they have different distribution for different $n$ as the law of the $\Phi_j^0$'s depends on $n$). Now, thanks to the fact that $\phi$ is Lipshitz continuous, we have that $\limsup_{n\to\infty} \hat \V[H_j]<\infty$. Hence, by a standard application of the CLT for triangular arrays, we get that
$$H-\hat\E[H] = \frac{1}{n}\sum_{j=1}^n(H_j-\hat \E[H_j])\to 0$$
in distribution, and hence in probability, as $0$ is a constant. It is quickly verified that the limit $h = \lim_{n\to\infty}\hat\E[H]$ exists, finite and positive. The proof for $\Theta$ is analogous.
\end{proof}

Now we can easily prove Proposition \ref{prop:init}.
\begin{manualprop}{\ref{prop:init}}
Consider a sequence of networks of increasing width initialised according to \eqref{eq:init}, and whose activation function $\phi$ is Lipshitz continuous. For any fixed input $x\neq 0$, defining $B$ as in Proposition \ref{prop:bentkus}, we have $\frac{B(\emme,\esse)}{\sqrt n}\to 0$, as $n\to\infty$, in probability with respect to the random initialisation $\hat\Pp$. More precisely, $B(\emme,\esse) = O(1)$ wrt $\hat\Pp$, as $n\to\infty$. In particular, at the initialisation the network tends to a Gaussian limit, in distribution wrt the intrinsic stochasticity $\Pp$ and in probability wrt $\hat\Pp$.
\end{manualprop}
\begin{proof}
It is a straight consequence of Lemma \ref{lemma:h_th}.
\end{proof}
\subsection{Lazy training}
We have established that the Gaussian limit holds at initialisation. In the present section we will see that, as far as the hyper-parameters of the network do not move too much from their initial values, the limit keeps its validity.
\begin{manualprop}{\ref{prop:lazytrain}}
Fix a constant $J>0$ independent of $n$, and assume that $\phi$ is Lipshitz. For a network of width $n$, with initial configuration $(\widetilde\emme,\widetilde\esse)$ drawn according to $\hat\Pp$ as in \eqref{eq:init}, denote as $\Bb_J$ the ball
$$
\Bb_J = \left\{(\emme,\esse)\;:\quad\|\emme^0-\widetilde\emme^0\|^2_{F,2} + \|\emme^1-\widetilde\emme^1\|_{F,2}^2+\|\esse^0-\widetilde\esse^0\|^2_{F,2} + \|\esse^1-\widetilde\esse^1\|_{F,2}^2\leq J^2\right\}\,,
$$
where $\|\cdot\|_{F,2}$ denotes the 2-Frobenius norm of a matrix.
Let $B$ be defined as in Proposition \ref{prop:bentkus}. For any fixed input $x\neq 0$ we have $B(\emme, \esse)= O(1)$ as $n\to\infty$, uniformly on $\Bb_J$, in probability with respect to the random initialisation $\hat\Pp$.
\end{manualprop}
%\todo[inline]{GD: here you mean that the $B$ corresponding to a network $(\mathfrak{m}, \mathfrak{s})$ is $O(1)$. When you introduce the tilde notation above, perhaps give an example, i.e. "for example we will write $\tilde{B}$ for formula (*) applied to the network at initialisation and $B$ for formula (*) applied after training. or sth similar. }
\begin{proof}
For convenience we will write with a tilde all the quantities relative to the network at initialisation. We denote with a $\Delta$ the difference between the final and the initial values of these quantities. For instance, $\Theta = \frac{1}{n}\sum_{j=1}^n \E[\Phi^0_j(x)^2]\min_{i=1\dots q}(\esse^1_{ij})^2$, $\widetilde\Theta = \frac{1}{n}\sum_{j=1}^n \E[\widetilde\Phi^0_j(x)^2]\min_{i=1\dots q}(\widetilde\esse^1_{ij})^2$, and $\Delta \Theta = \Theta - \widetilde\Theta$.

We will show that for $n\to\infty$, $\Theta=  \Omega_{\hat\Pp}(1)$ and $G= O_{\hat\Pp}(1)$ uniformly on $\Bb_J$, so that we can conclude using that $\Lambda\geq\Theta$ and Lemma \ref{lemma:Bbound}.

Fix an input $x$. First, we need a bound on $\|\Delta\Phi^0(x)\|_2 = \|\Phi^0(x) - \widetilde\Phi^0(x)\|_2$. We have that $\Phi^0(x) = \phi(Y^0(x))$. Hence, letting $L$ be the Lipshitz constant of $\phi$, we have
$\|\Delta\Phi^0(x)\|_2 \leq L\|\Delta Y^0(x)\|_2$. Now, as $\Delta Y^0_j(x) = \frac{1}{\sqrt p}\sum_{k=1}^p\Delta\emme^0_{jk}x_k + \frac{1}{\sqrt p}\sum_{k=1}^p \Delta\esse^0_{jk}\zeta^0_{jk}x_k$, we have
$$\|\Delta \Phi^0(x)\|_2\leq\frac{L}{\sqrt p}(\|\Delta\emme^0\|_2 + \|\Delta\esse^0\odot\zeta^0\|_2)\|x\|_2 \,,$$
where $\odot$ denotes the Hadamard product.

Notice that we have
$$\E[\|\Delta\esse^0\odot\zeta^0\|^2_2]\leq \E[\|\Delta\esse^0\odot\zeta^0\|_{F,2}^2] =\sum_{j=1}^n\sum_{k=1}^p (\Delta\esse^0_{jk})^2\E[(\zeta^0_{jk})^2] = \|\Delta\esse^0\|_{F,2}^2\leq J^2$$
uniformly in $\Bb_J$, where as usual the expectation $\E$ is the one with respect to the intrinsic stochasticity of the network, due to the $\zeta$'s.
%\todo[inline]{Is this expectation wrt the Gaussianity then? need notation to distinguish between the two. }
We can define a constant $C\geq 0$, independent of $n$, such that
$$\E[\|\Delta\Phi^0(x)\|_2^2]\leq\frac{4L^2J^2\|x\|_2^2}{p} = C^2$$
uniformly in $\Bb_J$, as $\|\Delta\emme^0\|_2\leq \|\Delta\emme^0\|_{F,2}\leq J$. 

Now, recalling the definition of $\Theta$ and using that $\esse^1 = 1 +\Delta \esse^1$, we have
$$\Theta = \frac{1}{n}\sum_{j=1}^n\E[\Phi^0_j(x)^2]\min_{i=1\dots q}(\esse^1_{ij})^2\geq \frac{1}{n}\sum_{j=1}^n\E[\Phi^0_j(x)^2](1 - 2\min_{i=1\dots q}|\Delta\esse^1_{ij}|) \,.$$
We will show that $\frac{1}{n}\sum_{j=1}^n\E[\Phi^0_j(x)^2]\to \widetilde\Theta$ and $\frac{1}{n}\sum_{j=1}^n\E[\Phi^0_j(x)^2]\min_{i=1\dots q}|\Delta\esse^1_{ij}|\to 0$.

First notice that 
$$\left|\frac{1}{n}\E[\|\Phi^0(x)\|_2^2] - \frac{1}{n}\E[\|\widetilde\Phi^0(x)\|_2^2]\right| \leq \frac{2}{n}\E[|\widetilde\Phi^0(x)\cdot\Delta \Phi^0(x)|] +\frac{1}{n}\E[\|\Delta\Phi^0_j(x)\|_2^2]\,.$$
We know that $\frac{1}{n}\E[\|\widetilde\Phi^0(x)\|_2^2]=\widetilde\Theta$ by definition. On the other hand we have
\begin{align*}\frac{2}{n}\E[|\widetilde\Phi^0(x)\cdot\Delta \Phi^0(x)|]&\leq\frac{2}{n}\E[\|\widetilde\Phi^0(x)\|_2\|\Delta\Phi^0(x)\|_2]\\&\leq \frac{2 C}{\sqrt n}\left(\frac{1}{n}\E[\|\widetilde\Phi^0(x)\|_2^2]\right)^{1/2}=\frac{2 C\widetilde\Theta^{1/2}}{\sqrt n}=O_{\hat\Pp}(1/\sqrt n) \,.\end{align*}
%\todo[inline]{Maybe $O_\mathbb{P}(1/\sqrt{n})$?}
Since $\frac{1}{n}\E[\|\Delta\Phi^0_j(x)\|_2^2]\leq C^2/n$, we have that $\frac{1}{n}\sum_{j=1}^n\E[\Phi^0_j(x)^2]-\widetilde\Theta\to 0$ uniformly in $\Bb_J$, in probability with respect to the random initialisation $\hat\Pp$.

We still need to show that $\frac{1}{n}\sum_{j=1}^n\E[\Phi^0_j(x)^2]\min_{i=1\dots q}|\Delta\esse^1_{ij}|\to 0$. Again we can decompose the term in $\Phi^0$ and we have
\begin{align*}\frac{1}{n}\sum_{j=1}^n\E[\Phi^0_j(x)^2]&\min_{i=1\dots q}|\Delta\esse^1_{ij}| \\&= \frac{1}{n}\sum_{j=1}^n(\E[\widetilde\Phi^0_j(x)^2]+2\E[\widetilde\Phi^0_j(x)\Delta\Phi^0_j(x)]+\E[\Delta\Phi^0_j(x)^2])\min_{i=1\dots q}|\Delta\esse^1_{ij}| \,.\end{align*}
Clearly, for every $j$ we have $\min_{i=1\dots q}|\Delta\esse^1_{ij}|\leq J$, and so we can write
\begin{align*}\frac{1}{n}\sum_{j=1}^n\E[\Phi^0_j(x)^2]\min_{i=1\dots q}|\Delta\esse^1_{ij}| &\leq \frac{1}{n}\sum_{j=1}^n\E[\widetilde\Phi^0_j(x)^2]\min_{i=1\dots q}|\Delta\esse^1_{ij}| \\&+ \frac{2J}{n}(\E[|\widetilde\Phi^0(x)\cdot\Delta\Phi^0(x)|] +\E[\|\Delta\Phi^0(x)\|_2^2])\end{align*}
uniformly in $\Bb_J$. We know already that $\frac{1}{n}(2\E[|\widetilde\Phi^0(x)\cdot\Delta\Phi^0(x)|] +\E[\|\Delta\Phi^0(x)\|_2^2])=O_{\hat\Pp}(1/\sqrt n)$. As for the other term, we have
$$\frac{1}{n}\sum_{j=1}^n\E[\widetilde\Phi^0_j(x)^2]\min_{i=1\dots q}|\Delta\esse^1_{ij}| \leq \frac{1}{\sqrt n}\left(\frac{1}{n}\sum_{j=1}^n\E[\widetilde\Phi^0_j(x)^4]\right)^{1/2}\left(\sum_{j=1}^n\min_{i=1\dots q}(\Delta\esse^1_{ij})^2\right)^{1/2}\,.$$
Using an argument analogous to that in the proof of Proposition \ref{prop:init}, we have that $\frac{1}{n}\sum_{j=1}^n\E[\widetilde\Phi^0_j(x)^4]$ has a finite limit (in probability wrt $\hat\Pp$). On the other hand, we have
$$\sum_{j=1}^n\min_{i=1\dots q}(\Delta\esse^1_{ij})^2 \leq \sum_{j=1}^n\sum_{i=1}^q(\Delta\esse^1_{ij})^2 = \|\Delta\esse^1\|_{F,2}^2\leq J^2\,.$$
We have thus obtained that $\frac{1}{n}\sum_{j=1}^n\E[\widetilde\Phi^0_j(x)^2]\min_{i=1\dots q}|\Delta\esse^1_{ij}|=O_{\hat\Pp}(1/\sqrt n)$, and so we can conclude that $\Theta = \Omega_{\hat\Pp}(1)$, uniformly in $\Bb_J$ and in probability wrt the random initialisation $\hat\Pp$. 

Now, we will show that $G= O_{\hat\Pp}(1)$. We have
$$G = \frac{1}{n}\sum_{j=1}^n\E[\|X_j\|_2^3]\leq \frac{4}{n}\sum_{j=1}^n\E[\|\widetilde X_j\|_2^3] + \frac{4}{n}\sum_{j=1}^n\E[\|\Delta X_j\|_2^3]\,.$$
Let us write $X_j = U_j + V_j$, with $U_j = (\zeta^1_{ij}\Phi^0_j(x))_{i=1\dots q}$ and $V_j = (\emme^1_{ij}(\Phi^0_j(x)-\E[\Phi^0_j(x)]))_{i=1\dots q}$. Then $\|\Delta X_j\|_2^3 \leq 4(\|\Delta U_j\|_2^3 + \|\Delta V_j\|_2^3)$.

First, denoting as $\zeta^1_j$ and $\Delta\esse^1_j$ the vectors $(\zeta^1_{ij})_{i=1\dots q}$ and $(\Delta\esse^1_{ij})_{i=1\dots q}$, we can write 
$$\Delta U_j = \Delta\Phi^0_j(x)\zeta^1_j +\widetilde\Phi^0_j(x) \Delta\esse^1_j\odot\zeta^1_j + \Delta\Phi^0_j(x)\Delta\esse^1_j\odot\zeta^1_j\,,$$
where $\odot$ represents the Hadamart product. $\Phi^0$ and $\zeta^1$ are independent and $\E[|\zeta|^3] = 2\sqrt{2/\pi}<2$ for $\zeta\sim\mathcal N(0,1)$, so we have
$$\E[\|\Delta U_j\|_2^3]\leq 54(q^{3/2}\E[|\Delta\Phi^0_j(x)|^3] + \E[|\widetilde\Phi^0_j(x)|^3]\E[\|\Delta\esse^1_j\|_2^3] + \E[|\Delta\Phi^0_j(x)|^3]\E[\|\Delta\esse^1_j\|_2^3])\,.$$
Using that $\|\Delta\Phi^0(x)\|_3^3\leq \|\Delta\Phi^0(x)\|_2^3\leq C^3$, we have that 
$$\frac{1}{n}\sum_{j=1}^nq^{3/2}\E[|\Delta\Phi^0_j(x)|^3]\leq \frac{q^{3/2}C^3}{n}$$ uniformly in $\Bb_J$.
Then we can notice that $\|\Delta\esse^1_j\|_2 \leq \|\mathbf 1_q\|_3\|\Delta\esse^1_j\|_6 = q^{1/3}\|\Delta\esse^1_j\|_6$ by H\"older's inequality. Hence
\begin{align*}\frac{1}{n}\sum_{j=1}^n \E[|\widetilde\Phi^0_j(x)|^3]\E[\|\Delta\esse^1_j\|_2^3]&\leq\frac{q\|\Delta\esse^1\|_{F,6}^3}{\sqrt n}\left(\frac{1}{n}\sum_{j=1}^n\E[|\widetilde\Phi^0_j(x)|^6]\right)^{1/2}\\&\leq \frac{qJ^3}{\sqrt n}\left(\frac{1}{n}\sum_{j=1}^n\E[|\widetilde\Phi^0_j(x)|^6]\right)^{1/2}=O_{\hat\Pp}(1/\sqrt n)\end{align*}
uniformly in $\Bb_J$, where the last equality comes from the usual argument that $\frac{1}{n}\sum_{j=1}^n\E[|\widetilde\Phi^0_j(x)|^6]$ has a finite limit in probability (with respect to the random initialisation).

Finally, we can notice that $|\Phi^0_j(x)|\leq\|\Phi^0(x)\|_2\leq C$ for all $j$, so that
$$\frac{1}{n}\sum_{j=1}^n\E[|\Delta\Phi^0_j(x)|^3]\E[\|\Delta\esse^1_j\|_2^3]\leq \frac{q^{1/2}}{n}C^3\|\Delta\esse^1\|_{F,3}^3\leq\frac{q^{1/2}C^3J^3}{n}$$
uniformly in $\Bb_J$, where we used that $\|\Delta\esse^1_j\|_2 \leq \|\mathbf 1_q\|_6\|\Delta\esse^1_j\|_3 = q^{1/6}\|\Delta\esse^1_j\|_3$ by H\"older's inequality, and that $\|\Delta\esse^1\|_{F,3}\leq \|\Delta\esse^1\|_{F,2}\leq J$. We can hence conclude that 
$$\frac{1}{n}\sum_{j=1}^n\E[\|\Delta U_j\|_2^3] = O_{\hat\Pp}(1/\sqrt n)$$
uniformly in $\Bb_J$.

Now we need to bound $\|\Delta V_j\|_2$. Letting $\emme^1_j = (\emme^1_{ij})_{i=1\dots q}$ and $\delta\Phi^0_j(x) = \Phi^0_j(x)-\E[\Phi^0_j(x)]$, it can be easily shown that
$$\|\Delta V_j\|_2 \leq |\delta\widetilde\Phi^0_j(x)|\|\Delta \emme^1_j\|_2 + |\Delta\delta\Phi^0_j(x)|\|\widetilde \emme^1_j\|_2 + \|\Delta \emme^1_j\|_2|\Delta\delta\Phi^0_j(x)|\,.$$
So we have
\begin{align*}\frac{1}{n}\sum_{j=1}^n\E[\|V_j\|_2^3]&\leq \frac{27}{n}\sum_{j=1}^n\E[|\delta\widetilde\Phi^0_j(x)|^3]\|\Delta \emme^1_j\|_2^3\\&+\frac{27}{n}\sum_{j=1}^n\E[|\Delta\delta\Phi^0_j(x)|^3]\|\widetilde \emme^1_j\|_2^3+\frac{27}{n}\sum_{j=1}^n\E[|\Delta\delta\Phi^0_j(x)|^3]\|\Delta \emme^1_j\|_2^3\,.\end{align*}

Starting from the first term, we have that
$$\sum_{j=1}^n\E[|\delta\widetilde\Phi^0_j(x)|^3]\|\Delta \emme^1_j\|_2^3\leq \left(\sum_{j=1}^n\E[|\delta\widetilde\Phi^0_j(x)|^6]\right)^{1/2} \left(\sum_{j=1}^n\|\Delta\emme^1_j\|_2^6\right)^{1/2}\,.$$
From H\"older's inequality we have $\|\Delta\emme^1_j\|_2\leq\|\mathbf 1_q\|_3\|\Delta\emme_j\|_6= q^{1/3}\|\Delta\emme_j\|_6$ and hence
\begin{align*}\frac{1}{n}\sum_{j=1}^n\E[|\delta\widetilde\Phi^0_j(x)|^3]\|\Delta \emme^1_j\|_2^3&\leq\frac{q\|\Delta\emme^1\|_{F,6}^3}{\sqrt n}\left(\frac{1}{n}\sum_{j=1}^n\E[|\delta\widetilde\Phi^0_j(x)|^6]\right)^{1/2}\\&\leq \frac{qJ^3}{\sqrt n}\left(\frac{1}{n}\sum_{j=1}^n\E[|\delta\widetilde\Phi^0_j(x)|^6]\right)^{1/2}=O_{\hat\Pp}(1/\sqrt n)\end{align*}
uniformly in $\Bb_J$, as $\frac{1}{n}\sum_{j=1}^n\E[|\delta\widetilde\Phi^0_j(x)|^6]$ tends in probability (wrt the random initialisation) to a finite limit.

Proceeding analogously, and noting that the $L$-Lipschitzianity of $\phi$ implies that $\E[|\Delta\delta\Phi^0_j(x)|^3]\leq 8L^3\E[|\Delta Y^0_j(x)|^3]$, we get
$$\frac{1}{n}\sum_{j=1}^n\E[|\Delta\delta\Phi^0_j(x)|^3]\|\widetilde \emme^1_j\|_2^3\leq \frac{8C^3}{\sqrt n}\left(\frac{1}{n}\sum_{j=1}^n\E[\|\widetilde\emme^1_j\|_2^6]\right)^{1/2}=O_{\hat\Pp}(1/\sqrt n)$$
uniformly in $\Bb_J$, and again we used the fact that $\frac{1}{n}\sum_{j=1}^n\E[\|\widetilde\emme^1_j\|_2^6]$ converges in probability (wrt the random initialisation) to a finite quantity to show that the above expression is of order $O_{\hat\Pp}(1/\sqrt n)$. 

Finally, in a similar way we get
$$\frac{1}{n}\sum_{j=1}^n\E[|\Delta\delta\Phi^0_j(x)|^3]\|\Delta \emme^1_j\|_2^3\leq \frac{8qJ^3C^3}{n}$$
uniformly in $\Bb_J$. We can hence conclude that $$\frac{1}{n}\sum_{j=1}^n\E[\|V_j\|_2^3]= O_{\hat\Pp}(1/\sqrt n)$$
and so that, as $n\to\infty$, $G\leq O_{\hat\Pp}(1)$, uniformly in $\Bb_J$ and in probability with respect to the random initialisation. This ends the proof.
\end{proof}
\section{Multiclass classification (\texorpdfstring{$q>2$}{q>2})}\label{app:multiclass}
In the framework of Section \ref{sec:01training}, things get more complicated when there are more than two classes. We can write 
$$\E[\ell(\hat f(x),i^\star)] = \Pp\left(F_{i^\star}(x) \leq \max_{i\neq i^\star} F_{i}(x)\right) = 1-\Pp\left(F_{i^\star}(x)> \max_{i\neq i^\star} F_i(x)\right)\,.$$
Hence, given a $q$-dimensional Gaussian vector $Y\sim \mathcal N(M, Q)$, we need to find an estimate of $\Pp(Y_{i^\star}>\max_{i\neq i^\star}Y_i)$.

The most trivial estimator would consist of sampling different realisations of $Y$ and then give a MC estimate. However, as we are interested in the gradient of the expected loss, this method will not work. Indeed, the gradient of this estimate is the sum of the gradients of the 01-loss of each sample. As all these gradients are null, we do not obtain anything informative. We have thus to proceed in a less naive way.

Let us assume that $i^\star = q$ (the largest label). Hence, we will focus on $\Pp(Y_q> \max_{i<q} Y_i)$. With a Cholesky-like algorithm,  we can find a lower triangular matrix $A$ such that $Y \sim A X + M$,
where $X\sim \mathcal N(0, \text{Id})$. We have $Y_i = \sum_{i'=1}^{q} A_{ii'}X_{i'} + M_i$ and $A_{iq}=0$ for $i=1\dots (q-1)$, while $A_{qq}>0$. For $i<q$, we can write
$$\Pp(Y_q>Y_i) = \Pp\left(X_q>\sum_{i'=1}^{q-1} \frac{A_{ii'}-A_{qi'}}{A_{qq}}\,X_{i'} + \frac{M_i-M_q}{A_{qq}}\right)\,.$$
Let us define the $(q-1)$ dimensional random vector $\tilde X$ as $\tilde X = \tilde A X + \tilde M$, where $\tilde A$ is a $(q-1)\times q$ matrix and $\tilde M$ is a $(q-1)$ vector, whose elements are given by $\tilde A_{ii'} = \frac{A_{ii'}-A_{qi'}}{A_{qq}}$ and  $\tilde M_i = \frac{M_i-M_q}{A_{qq}}$ repectively. With this notation, we have $\Pp(Y_q>Y_i) = \Pp(X_q>\tilde X_i)$. Now, we have gained that $X_q$ is independent from all the other $X_i$'s, and so from all the $\tilde X_i$'s. In short, $(X_q|\tilde X) = X_q\sim \mathcal N(0,1)$. So, we can write 
$$\Pp\left(Y_q>\max_{i<q} Y_i\right) = \Pp\left(X_q>\max_{i<q} \tilde X_i\right) = \E\left[\Pp\left(X_q>\max_{i<q} \tilde X_i\bigg|\tilde X\right)\right]\,.$$
Now, if we let $\psi(u) = \tfrac{1}{2}(1-\text{erf}\,(u/\sqrt 2))$, we get
$$\Pp\left(Y_q>\max_{i<q} Y_i\right) = \E\left[\psi\left(\max_{i<q} \tilde X_i\right)\right]\,.$$
We can estimate the above expression with MC sampling. Note that it is almost everywhere differentiable with respect to the components of $M$ and $Q$ (as the Cholesky transform is differentiable) and the gradient with respect to $M$ and $Q$ is not trivially null.

Finally, for a general $i^\star\in\{1,\dots, q
\}$, we can get $\Pp(Y_{i^\star}>\max_{i\neq i^\star}Y_i)$ by simply performing a swap of the two labels $i^\star$ and $q$, and then apply the method for $i^\star=q$. 

\section{Expected values for \texorpdfstring{$\ReLU$}{ReLU} and \texorpdfstring{$\sin$}{sin} activations}\label{app:relusin}
Let $a>0$, $b\in\R$, $\zeta\sim\mathcal N(0,1)$. The following formulae are easily verified by direct calculations:
\begin{align*}
    &\E[\sin(a\zeta + b)] = e^{-a^2/2}\sin b\,;\\
    &\E[\sin(a\zeta + b)^2] = \frac{1}{2}(1-e^{2a^2}\cos(2b))\,;\\
    &\E[\ReLU(a\zeta + b)] = \frac{ae^{-b^2/(2a^2)}}{\sqrt{2\pi}}+\frac{b}{2}\left(1+\erf\frac{b}{a\sqrt 2}\right)\,;\\
    &\E[\ReLU(a\zeta + b)^2] = \frac{abe^{-b^2/(2a^2)}}{\sqrt{2\pi}} + \frac{1}{2}(a^2+b^2)\left(1+\erf\frac{b}{a\sqrt 2}\right)\,.
\end{align*}
\section{Experimental details}\label{app:exp}
In all the experiments, the training consisted of optimising some PAC-Bayesian bound via SGD with momentum parameter $0.9$. The PAC parameter $\delta$ was always chosen equal to $0.025$. We only performed the training of the means $\emme$ and all the networks considered had no bias. The priors corresponded to the initialisation of the network \eqref{eq:init}. Note that in our implementation, the scaling factors $1/\sqrt p$ and $1/\sqrt n$ were absorbed in the hyper-parameters, so that we performed the gradient descent on $\mu^0 = \emme^0/\sqrt p$ and $\mu^1 = \emme^1/\sqrt n$ (the standard deviations were kept fixed).

For the binary MNIST experiments, the digits from $0$ to $4$ were relabelled as $0$ and those from $5$ to $9$ as $1$. The training dataset used was the standard one for MNIST, consisting of $m=60000$ datapoints. For the ``standard'' PAC-Bayesian methods, the objectives used are those reported in \eqref{eq:objs}. For the objective $\mathtt{lbd}$ we proceeded by alternating the optimisation of the network hyper-parameters with that of $\lambda$, as in \cite{perezortiz2021tighter}, always enforcing $\lambda\in(0,1)$. The ``Gaussian'' training was performed with the optimisation objective \eqref{eq:objbound}. All of these methods were used to train the same stochastic network, initialised as in \eqref{eq:init}. We tried two different learning rate (LR) schedules, the first consisting of $10000$ epochs with LR $\eta=10^{-5}$ and the second of $100$ epochs with $\eta=10^{-2}$, followed by $1000$ epochs with $\eta=10^{-3}$ and $5000$ epochs with $\eta=10^{-4}$. In Table \ref{tab:binmnist} in the main text we report the results of the training schedule achieving the tightest bound, that is the multi-LR schedule for $\mathtt{invkl}$ and $\mathtt{quad}$, and the single-LR schedule for $\mathtt{McAll}$ and $\mathtt{lbd}$.

For the full MNIST experiments, again we used the standard training dataset with $m=60000$ datapoints. For the ``standard'' methods, $L_S(\Q)$ in \eqref{eq:objs} was a bounded version of the cross-entropy: we fixed $p_0=10^{-5}$ and constrained the probabilities in the definition of the cross-entropy to be greater or equal than $p_0$, see \citep{perezortiz2021tighter} for more details. In this way, the loss is bounded by $\log(1/p_0)$, and by rescaling it of the same factor we can get a loss bounded in $[0,1]$. $L_S(\Q)$ is the empirical average of this quantity on the training dataset. As we previously did for the binary MNIST experiment, during the training we estimated $L_S(\Q)$ by sampling once per iteration the hyper-parameters of the network. The ``Gaussian'' method used the objective \eqref{eq:objbound}, where $L_S(\Q)$ is the expected empirical 01-loss. As we are dealing with multiclass classification we do not have a simple expression for the 01-loss, so we used the method described in Appendix \ref{app:multiclass}. Per each iteration, the loss was evaluated by an MC estimate averaging $10^4$ independent realisations. For all the methods, the training consisted of $10000$ epochs with learning rate $\eta=10^{-5}$.

\end{document}